\documentclass{article}

\usepackage{amsfonts,amsmath,amssymb}
\usepackage{listings} 
\usepackage{hyperref}
\usepackage{enumerate}
\usepackage{latexsym}
\usepackage{fancyhdr}
\usepackage{ifthen}
\usepackage{tabularx}
\usepackage{dcolumn}
\usepackage{tikz}
\usepackage{caption}
\usepackage{graphicx,subfigure,float}
\usepackage{amsthm}
\usepackage{adjustbox}
\usepackage[sort&compress]{natbib}

\newcommand{\Prob}{\mathcal{P}rob}
\newcommand{\Stat}{\mathcal{S}tat}
\newcommand{\Bayes}{\mathcal{B}ayes}
\newcommand{\Meas}{\mathcal{M}eas}
\newcommand{\R}{\mathbb{R}}
\newcommand{\Tr}{\mathrm{Tr}}
\newcommand{\LV}{\mathrm{LV}}
\newcommand{\V}{\mathrm{V}}
\newcommand{\C}{\mathrm{C}}
\newcommand{\NN}{\mathrm{NN}}
\newcommand{\TS}{\mathrm{TS}}

\newcommand{\Loc}{\mathrm{Loc}}

\def\comma{}

\title{To Describe or Construct Statistical Learning Models Using the Category-theoretical Language}  
\author{Congwei Song}  
\date{\today}  

\theoremstyle{plain}

\newtheorem{fact}{Fact}[section]

\newtheorem{definition}{Definition}[section]
\newtheorem{remark}{Remark}[section]

\begin{document}  
\maketitle

%%% ----------------------------------------------------------------------
% \newenvironment{proof}{\textit{Proof}\hspace{0.5cm}}{$\Box$}

\textbf{Abstract} Statistical learning is a fascinating field that has long been the mainstream of machine learning/artificial intelligence. A large number of results have been produced which can be widely applied to real-world problems. It also leads to many research topics and also stimulates new research. This report summarizes some classical statistical learning models and well-known algorithms, especially for amateurs, and provides a category-theoretic perspective on understanding statistical learning models. The aim is to attract researchers from other fields, including basic mathematics, to participate in the research related to statistical learning.


\section{Introduction}

People have developed plentiful machine learning models, from the simple, such as linear regression, logistic regression\cite{hastie2009}, to the complicated, such as deep neural networks and large language models\cite{goodfellow2016}. However, the models share with the similar pattern as statistic models. In this report, we will investigate them preliminarily with the help of category theory.

The main purpose of the paper is not to build a strict theory as in \cite{shiebler2021}, but to invent some notaions borrowed from category theory to describe complicated models of machine learning.

The two primary contributions of this paper are:
\begin{enumerate}[(1)]
\item The utilization of categorical language to describe the models of machine learning.
\item The introduction of a unified descriptive approach to construct complicated models, such as the language models.
\end{enumerate}

As a nontrivial example, the Transformer is constructed based on the so-called local model or localization trick.

We introduce the following abbreviations and notations.

% \textbf{Abbreviations}:
% \begin{itemize}
%   \item rv(s): random variable(s)
%   \item w.r.t.: with respect to
% \end{itemize}

\textbf{Abbreviations for classical models}:
\begin{itemize}
  \item LDA/QDA: Linear Discriminant Analysis/Quadratic Discriminant Analysis
  \item GMM: Gaussian Mixture Model
  \item PCA: Principal Component Analysis
  \item ICA: Independent Component Analysis
  \item NMF: Non-negative Matrix Factorization
  \item (p)LSA: (Probabilistic) Latent Semantic Analysis
  \item HMM: Hidden Markov Model
  \item RNN: Recurrent Neural Network
  \item LSTM: Long-Short time memory (as a classical design of RNN)
\end{itemize}

\textbf{Notations}:
\begin{itemize}
  \item $\mathcal{X}$: sample space
  \item $\mathcal{X}^*=\bigcup_t \mathcal{X}^t$: sequential sample space (Kleen closure of $\mathcal{X}$)
  \item $X$: a random variable (rv)
  \item $\{X_i\}$: a sample
  \item $X^*$: a sequence of rv with arbitary length
  \item $P(X)$: distribution of target rv $X$
  \item $P(X|\theta)$: parametric distribution
  \item $\{x_\lambda\}$: a family/sequence/set indexed by $\lambda$
  \item $\{x_{ij}\}_{ij}$: a matrix indexed by $i,j$
  \item $A := B$: $A$ is defined as $B$
  \item $f(x;y)$: A function of $x$ and $y$, while $y$ is not important in theoretical analysis and will be omitted in common cases
\end{itemize}

In the context of machine learning, we use lowercase $x$ to represent rvs.

Furthermore, the ``facts'' in this paper refer to theorems that are not completely formalized. However, they can be formalized at any time.

\section{Category theory}

We do note need the entire category theory. Here, it is merely used as a language for description machine learning.

The category theory studies objects, morphisms/arrows between them, and functors as a mapping between cats. For simplicity, we will only consider a category as a family of objects and view functors as a mapping from one family to another, without explicitly discussing arrows which is the real thrust of category theory. An object is referred to as a model, and a functor is also called a ``constructor'' in this paper.

We will not recall the definition of category theory here, and for more details about category theory, please refer to \cite{maclane1978}.

\section{Probability Models}

Let us depart from the definition of probability models. Probability theory is truly built on measure theory, where distributions play a crucial role \cite{jost2019,khoshnevisan2007,kiessler1998}.

\begin{definition}[Probability Model 1]
  A probability model is a special measure space, denoted by $(\Omega, \mathcal{A}, P)$ where the probaility measure $P$ satifise $P(\Omega)=1$, $\Omega$ is the ambient space, and $\mathcal{A}$ is a $\sigma$-algebra on $\Omega$.
\end{definition}

A measurable mapping $X$ on $\Omega$ is called $\mathcal{X}$-valued rv, if the range of $X$ is the measurable space $(\mathcal{X},\mathcal{B})$, also called the sample space of $X$.

We will not explain the measure space in details, since what we are interested in is the following special setting.

\begin{definition}[Probability Model 2]
  Assume that $(\Omega, \mathcal{A}, P)$ is any probability model. Then a probability model derived by the target rv $X:\Omega\to \mathcal{X}$ is
  \begin{equation}\label{pushback}
  (\mathcal{X},\mathcal{B}, P_X), P_X(B) := P(X\in B), B\in\mathcal{B}
  \end{equation}
  where $X\in B$ represent the event $\{\omega\in\Omega X(\omega)\in B\}$, namely the preimage $X^{-1}[B]$.
  \eqref{pushback} is written as $(\mathcal{X}, P(X))$ or just $P(X)$ for short.
\end{definition}

It is equivalent to say that the distribution of $X$ is $P$, or $X$ obeys the distribution $P$, 
denoted as $X\sim P$.

\begin{remark}
  The model \eqref{pushback} is named the \textit{pushback} of $X$ in measure theory \cite{shiebler2021}.
\end{remark}

\section{Statistics}

Statistics is the foundation of Statistical Learning/Machine Learning, whose research object is the statistical models.

\subsection{Statistical Models}

\begin{definition}[Statistical Model (non-parameterized)]\label{df:sm}
A statistical model is a family/set of probability models with the common ambient space denoted by $(\Omega, \mathcal{A},\{P_\lambda\})$. Of interest is the special case $(\mathcal{X},\{P_\lambda(X)\})$ that is a statistical model generated by the target rv $X$ whose common ambient space is the sample space of $X$.
\end{definition}

It is equivalent to write that $X\sim P_\lambda$ where $P_\lambda$ is unkown.

It is possible to parameterize a family/set $\{a_\lambda\}$ indexed by $\lambda$ through a mapping $\theta\mapsto a_\theta,\theta\in\Theta$, say that $a_\theta$ is the parameterized form of $\{a_\lambda\}$. Strictly speaking, the parameter space $\Theta$ is a subset of Eulidean space $\mathbb{R}^n$ \cite{shao2007}, but we will remove the restriction. Now we introduce the parameterized form of the statistical models in \autoref{df:sm}

\begin{definition}[Statistical Model (parameterized)]\label{df:sm-param}
  A parametric statistical model is the parameterized form of a statistical model (in \autoref{df:sm}), denoted as $(\mathcal{X},P(X|\theta),\theta\in\Theta)$ where $\Theta$ is the parameter space, and $P(X|\theta)$ is a mapping on $\Theta$ strictly. If there is no ambiguity, we denote it by the para distribution $P(X|\theta)$ without any explanation of parameter space.
\end{definition}

Although the parameterized form is used in machine learning more frequently, in theoretical analysis, nonparameteric statistical model. is commonly used in the paper. We use $P(X)$ to refer to a general statistical model where $P$ is unknown self-declaratively.

\begin{remark}
Since we did not give any restriction for parameter space, thus there is no essential difference between parameterized form and non-parameterized form of the statisical model, in this paper.
\end{remark}

\cite{mccullagh2002} provides a more strict definition of statistical models. The different is not essential. Even after this paper concludes, it won't be too late to replace the definition in this article with it.

\subsection{Sample}

A \textit{sample} is a tuple/set/family of rvs, $\{X_1,\cdots, X_N\}$, drawn from a certain \textit{population}. If the sample is independent, meaning the rvs in sample are independent, then the joint distribution $P(X_1,\cdots, X_N)$, is determined by the popluation distribution totally, otherwise we should give out the joint distribution explicitly whose marginal distribution is the population. In the latter case, we have consturcted a new statistical model on the sample space $\mathcal{X}^N$, based on the original population. To distinguish the sample of the new model from the model $M$, we denote it as $M^*$ for brevity.

A \textit{statistic} is a function on a sample, denoted by $T(X)$ where $X$ represtents a sample.

\begin{remark}
(The realization of) a sample is commonly referred to as ``data'' in the context of machine learning.
\end{remark}

\subsection{Bayesian Models}

Bayesian methods are a class of statistical techniques and principles used for probabilistic reasoning and statistical inference. The core idea of Bayesian methods is to update our beliefs about a particular parameter represented by a distribution based on the observants. Baysian Model is the statistical model using the method.

\begin{definition}[Baysian Model]
  The Baysian model is a statistical model with priori distribution of parameters, as $(M_\theta,p(\theta))$, where $M_\theta$ is a given statistical model with parameter distribution $p(x|\theta)$. Thus a Bayesian model is identified with a probability model $p(x,\theta)$.
\end{definition}

A common statistical model is a special Bayesian model with that $\theta$ obayes so-called flatten distribution. Thus any statistical model is indeed a probality model with joint disribution $p(x,\theta)$, where we have no sample for $\theta$.

It will be more flexible to introduce a unknown super-parameter $\alpha$ for the priori distribution. $p(\theta|\alpha)$.

\begin{definition}[Baysian Hierachical Model]
  The \textit{Baysian empirical model} is a Bayesian model with unknown priori, denoted as $(M_\theta,p(\theta|\alpha))$ where $\alpha$ is unknown. If the super-parameter $\alpha$ has its own distribution $p(\alpha)$ (named the super priori), then we have the \textit{Baysian hierachical model} $(M_\theta, p(\theta|\alpha), p(\alpha))$.
\end{definition}

\subsection{Stochastic process}

As mentioned above, when samples are not independent, we use the joint distribution of sample to characterize the model. It is fundamentally different from the model of population distribution. In another word, the sample forms a stochastic process.

Roughly, a stochastic process is a sequence $\{X_t\}$ of rvs $X_t,t\in\mathcal{T}$, where the time set $\mathcal{T}=[0,\infty)$ or $\{1,\cdots, T\}$ normally. The process is charactorized by the following family of distributions, 
\begin{equation}
  P(X_{k_1},\cdots, X_{k_T}),k_1,\cdots, k_T\in\mathcal{T}
\end{equation}

If the time set $\mathcal{T}=\{1,\cdots, T\}$, the sequence $\{X_t\}$ is also named the time series. We take the \textit{Markov chain} as a typical example. Its joint distribution is formulated as, 
\begin{equation}
  P(X_1,\cdots, X_T)=P(X_1)\prod_{t=1}^{T-1} P(X_{t+1}|X_{t})
\end{equation}

% It is easy to see that $P^*$ also denotes a stochastic process whose the marginial distribution of any rv $X_k$ is $P$.

\subsection{Categories of statistical models}

As is well-known, the collection of measure spaces forms a category (denoted as $\Meas$ \cite{shiebler2021}), as do the categories of probability models and statistical models. The following categories can be expressed based on the definitions above without any explanation where the arrows will not be shown as mentioned.

\begin{itemize}
  \item $\Prob$: the category of all probability models, as the sub-category of $\Meas$;
  \item $\Prob_X$: the sub-category of $\Prob$ derived by the target rv $X$;
  \item $\Prob_{Y|x}$: the sub-category of $\Prob$ derived by the target rv $Y$, with conditional distribution $P(Y|x)$ where $x$ is fixed, named the conditional variable;
  \item $\Stat$: the category of The statistical models;
  \item $\Stat_X$: the sub-category of $\Stat$ of the target rv $X$, where an object is also represented by $P(X)$ ($P$ is unknown)
  \item $\Stat_{Y|x}$: the sub-category of $\Stat$ of the target rv $Y$ with respect to the conditional variable $x$;
  \item $\Bayes$: the category of the Baysian models;
  \item $P^*$: the stochastic process or the sample distribution with marginal distribution $P$;
\end{itemize}

\begin{remark}
Here the \textit{sample distribution} refers to the joint distribution of the sample, instead of the empirical distribution.
\end{remark}
% If $x$ is not fixed, then $\Prob_{Y|x}$ is a family of probability models with the index $x$. Similar is $\Stat_{Y|x}$.

\begin{fact}\label{fc:stat-prob}
  $\Stat \simeq \Prob$, in some sense.
\end{fact}

\begin{proof}
  $\Stat$ could be regarded as a sub-category of $\Bayes$ with the flatten priori. Obversely the Bayesian model is determined by the joint distribution. $P(x,\theta)$. As a result, we have that the two categroies, $\Stat$ and $\Prob$, are identified with each other in some sense.
\end{proof}

\begin{remark}
  The conditional distribution $P(Y|x)$ with a given $x$ is also a distribution. There is no different between $\Prob_{Y|x}$ and $\Prob_{X}$ as categories.
\end{remark}

\subsection{Estimator}

The main task of the statistcs is the estimation, in parameterized form or nonparameterized form. Formally, An estimation of a statistical model $P(X|\theta)$ is represented as mapping $\hat\theta: \mathcal{X}^N\to \Theta$.

\begin{definition}[Estimator]
  An estimator is a statistic to estimate the parameter of the model $P(X|\theta)$, denoted as $\hat\theta(X)$. A model with estimator is denoted as $(M_\theta, \hat\theta(S))$ where $S$ is a sample with size $N$.
\end{definition}

The best known estimator is the \textit{maximal likelihood estimator (MLE)}, which has been used to construct the most algorithms in machine learning.

\begin{definition}[MLE]
  Given a statistical model $P(X|\theta)$. The solution of the following optimization problem is called MLE of the parameter $\theta$:
  \begin{equation}
  \max_\theta l(\theta):=\sum_i \ln p(x_i|\theta)
  \end{equation}
  or in nonparameterized form,
  \begin{equation}
  \max_p l(p):=\sum_i \ln p(x_i) \comma
  \end{equation}
  where $\{x_i\}$ is the independent sample.
\end{definition}

In most cases, we use MLE, so the estimator has been implied by a model. Meanwhile, the word ``estimator'' and ``model'' are treated as the synonyms.

\begin{definition}[Maximum a posteriori estimation]
  The solution of the following optimization problem is called maximum a posteriori estimation (MAPE) of $\theta$.
  \begin{equation}
  \max_\theta l(\theta):=\sum_i \ln p(x_i|\theta) + N\ln p(\theta) \comma
  \end{equation}
  where $\{x_i\}$ is the independent sample with the size $N$.
\end{definition}

We can redefine the category of statistical models as a tuple of a distribution family $P(X|\theta)$ and a sample $\{x_i\}$, that makes it different from the category of probability models.

\section{Statistical Learning}

Machine learning models are divided to two main subclasses: supervised (learning) models and unsupervised (learning) models. We accept this important partition in the field of statistical learning.

\subsection{supervised learning model}

In the classical realm of machine learning, A model generally represents the relation of its input variable $x$ taking values in the space $\mathcal{X}$, and its corresponding output variable $Y$ taking values in the space $\mathcal{Y}$. Formally we give the following rough definition of machine learning models.

\begin{definition}\label{df:ml}
A machine learning model is determined by a mapping $f:\mathcal{X}\to \mathcal{Y}$, could be written as $y\sim f(x)$ simply, without an exponential distribution of $x$ and $y$.
\end{definition}

People expect that the instances $(x_i,y_i)$ chould satisfy the equations $f(x_i)=y_i, i=1\cdots,N$, that derives the following optimization problem:
\begin{equation}\label{ml-opt}
  \min_f \sum_il(y_i,f(x_i)) \comma
\end{equation}
where $l$ is a loss function defined on $\mathcal{Y}\times \mathcal{Y}$.

In statistical learning, the supervised learning model is described by a distribution, and has two basic forms: the determinant form and the generative form.

The determinant models are the probabilitic form of the machine learning models.

\begin{definition}[Supervised learning model (determinant form)]
  A supervised learning model with the input variable $x$ and the output variable $y$ is exactly depicted by the unknown conditional distribution $P(y|x)$, and denoted as $(\mathcal{X},\mathcal{Y}, P(y|x))$. It is called the determinant model. It could be seen as a family of $Stat_{y|x}$ objects indexed by the observations of $x$.
\end{definition}

\begin{remark} The formal names of $X$ and $Y$ are the independent variable/endogenous variable/explanatory variable/regression factor/predictor), and dependent variable/exogenous variable/explained variable/regression object/predictor, respectively. But we tend to prefer using the colloquial expression, ``input variable'' and ``output variable''.
\end{remark}

\begin{fact}\label{fc:dm}
A determinant model based on the sample $\{(x_i,y_i)\}$, is totally identified with a statistical model: 
\begin{equation}\label{sample-sm}
  (\mathcal{Y}^N, P(Y|X)=\prod_iP(y_i|x_i))
\end{equation} 
where the sample $X=\{x_i\}$ is fixed, named the \textit{design variable} (\textit{design matrix} , if it forms a matrix) and $Y=\{y_i\}$ is a $\mathcal{Y}^N$-valued sample point of the model \eqref{sample-sm}.
\end{fact}

We shall maximize the conditional liklihood $\max_p l(p)=\ln p(y_i|x_i)$, named the \textit{conditional maximial likelihood estimation (CMLE)} that is not fundamentally different from MLE.

\autoref{fc:dm} is the reason why the author claims that the statistical learning (model) is identified with the conditionalized statistics (model).

The definition of generative form is more simple in the abstract view. 
\textit{supervised learning model (generative form)} 
The generative model is denoted as $(\mathcal{X},\mathcal{Y}, P(x, y))$ that is not different from the statistical model $(\mathcal{X}\times \mathcal{Y}, P(x, y))$.

\textit{supervised learning model (Bayesian form)} 
The Bayesian form of supervied learning model is a special but common case of generateive model, denoted as $(\mathcal{X},\mathcal{Y}, P(x|y), P(y))$ where $P(x|y)$ and $P(y)$ is the generative distribution and priori distribution of the model respectively. Here the key point is that the two distribution are designed seperatively.

\subsection{Unsupervised learning models}

\begin{definition}
In general, unsupervised learning models only have generative form:
$(\mathcal{X}, \mathcal{Z}, P(x, z))$ where $z$ is unobservable, that is, no sample of $z$ is observed.
\end{definition}

From the view of population, the unsupervised learning model is not distingusied with the generative supervised learning model, and from the view of sample, it is the same with the statistical model $(\mathcal{X}, P(x))$, where $P(x)$ is the marginal distribution of $P(x,z)$, but it always has the prediction task with $P(z|x)$. The MLE of $P(X)$ is the \textit{maximal marginal likelihood estimation (MMLE)} of the unsupervised learning model. As you see, the determinant model could not be solved by MMLE.

An unobservale rv is also called the \textit{latent variable}. A model with latent variables is called the \textit{latent variable model}. Hence the unsupervised learning models are the latent variable models.

\subsection{Bayesian statistical learning models}

Applying the Bayesian method into statistical learning, we get the Bayesian statistical learning models (also called Bayesian models). Let $M_\theta$ is the unsupervised learning model with the distribution $P(x,z|\theta)$, then the corresponding Bayesian model is the tuple $(M_\theta, P(\theta))$. As we know, It is characterized by the joint distribution $P(x,z,\theta)$.

Like \autoref{fc:stat-prob}, we have the following fact.
\begin{fact}\label{fc:stat-prob-sl}
The statistical learning models, statistical models and probability models are equivalent to each other in some sense.
\end{fact}

We consider any statistical learning model as a Bayesian model: $P(x,z,\theta)$ where the latent variable $Z$ and the parameter $\theta$ are both unobservable variables, could be integrated to one variable. It is conluded that, any statistical learning model has the following form,
\begin{equation}
  P(x,z)
\end{equation} 
where $z$ is unobservable.

Now, we have eliminated the boundaries of hidden variables and parameters. It implies that we can estimate the parameter $\theta$ and predict the latent variable $Z$ simultaneously, by solving the following optimization problem, 
\begin{equation}\label{momle}
  \max_{\theta,\{z_i\}}\sum_i\ln p(x_i,z_i,\theta)
\end{equation}
named the \textit{mode maximal likelihood estimation (MoMLE)}, instead of the notorious EM algorithm, where $X$ is the sample. In fact, that is what the classical matrix factorization algorithms do, such as PCA and MNF. Moreover, MoMLE could solve the determinant unsupervised learning models $P(z|x)$, by 
\begin{equation*}
\max_{\theta,\{z_i\}}\sum_i\ln p(z_i|x_i,\theta)
\end{equation*}

\subsection{Learner}

A learner is defined as an estimator for a statistical learning model: 
\begin{equation*}
  (M_\theta, \hat{\theta}(D))
\end{equation*} 
where $M_\theta$ is a statistical model, $D$ is the sample. By the same reason about the estimator, the learner is also the synonym of the ambient statistical model.

It is useful to write the the learner explicitly in some cases. 
For instance, One can define a latent variable model as a learner $(P(x,z), \hat{\theta}(X))$ for unsupervised learning; $(P(X,Y), \hat{\theta}(X,Y))$ for supervised learning.

\subsection{predictor and generator}

In real-world situations, people are most concerned with the task of predicting the output $y$, given the input $x$. A predictor is used to do such task, namely computing $\arg\max_y P(y|x)$ or $E(y|x)$.

A generator of statistical learning model is simulation of the distribution $P(x,y)$ fundamentally.

A full structure of statistical learning model is considered to include the predictor, even the generator. As most cases, models already contain them in an implicit manner.

\subsection{Classic Models in Statistical Learning}

Following lists the classical models in statistical learning, supplemented by common examples. Each of them form a category that will not be explained in details.

\begin{itemize}
  \item Supervised Model:
  \begin{itemize}
  \item Regression: $P(y|x)$, linear regression, ridge/LASSO regression
  \item Classification:
    \begin{itemize}
    \item Discriminative form: $P(y|x)$, logistic regression
    \item Generative form: $P(x, y)$, LDA/QDA/Naive Bayesian classifier
    \end{itemize}
  \end{itemize}
  \item Unsupervised Model/Latent Variable Models:
  \begin{itemize}
    \item Clustering: $P(x,z)$, where $Z$ is unobservable (hereafter), K-means/GMM
    \item Dimension reduction: $P(x,z)$, such as PCA/ICA/MNF/(p)LSA
    \item Time Series models: $P(X_{1:T}, Z_{1:T})$ where $(X_t,Z_t)\sim P$ or $P(X^*,Z^*)$ for arbitray length, such as HMM.
  \end{itemize}
  \item Semi-supervied Model: $P(X^{(0)}, Z)P(X^{(1)}, Y)$ (also as a latent variable model)
\end{itemize}

There must be series models under supervised settings, and Naive Bayesian classifiers also have unsupervised versions. However, such models are not as widely recognized or renowned in the study of machine learning.

\subsection{The beginners' Cube}

Let us incorporate the classical models as categories within the diagram in \autoref{fig:cube}, which I propose to name ``the beginners' cube''. This nomenclature is chosen due to its resemblance to a cube and the importance of beginners initially familiarizing themselves with the models contained in it.

\begin{figure}[h]
  \centering
  \includegraphics[scale=0.6]{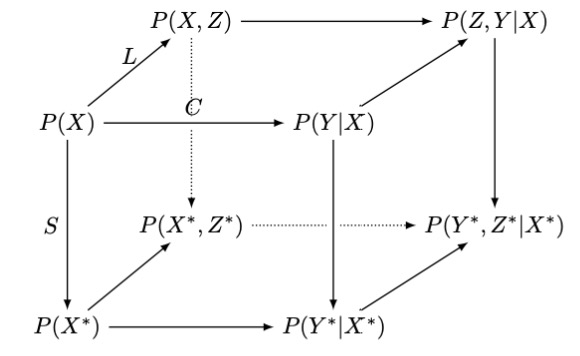}
  \caption{The beginner's cube}
  \label{fig:cube}
\end{figure}

The vertices in the diagram are categories of the models, but labeled by the generic form of their constituent objects. The arrows symbolize functors connecting any two given categories, with parallel arrows considered as equivalents. 

Consequently, we assign specific names to the arrows from $P(X)$ to $P(Y|X)$ or $P(X,Y)$, $P(x,z)$ and $P(X^*)$ respectively:
\begin{enumerate}[(1)]
  \item Conditionalizing (w.r.t. $x$), denoted by $\C_x$;
  \item Adding hidden/latent variable $z$, denoted by $\LV_z$;
  \item Serialization, denoted by $S$, or $\TS_z$ to emphasize that only the $z_t$ form a Markov chain;
\end{enumerate}

Hence we have $P(x,z)=L_z(P(x))$ for general latent variable models. Specially, if the hidden variable takes finite values in the latent varable model $P(x,z)$, then we replace $L_z$ with $Mix$, say, mixing distributions to construct a mixed model.

We have following simple facts:
\begin{enumerate}[(1)]
  \item $GMM = Mix(N(\mu,\Sigma))$
  \item $HMM = \TS_z P(x,z) = \TS_z \LV_z(x,z)$ and Gaussian-$HMM = \TS_z GMM$.
\end{enumerate}

\begin{remark}
  To enhance readability and reduce redundancy, we routinely omit subscripts in the constructors when there is no possibility of confusion.
\end{remark}

\subsection{Probabilitic Graphic Models}

Another way to describe the statistical (learning) model is the \textit{probabilitic graphic models (PGMs)}. There are two types of PGMs: the \textit{Bayesian network} and the \textit{Markov network }(also named the random field).

% A PGM describes a model in the graph-theoric way. We will not have to face with a mess of formula.

We will not recall the whole theory of PGMs here. instead we give an example of Bayesian network for HMM with joint distribution: $p(z_1)\prod_tp(x_t|z_t)\prod_t p(x_{t+1}|z_t)$.

\begin{figure}[h]
  \centering
  \includegraphics[scale=0.8]{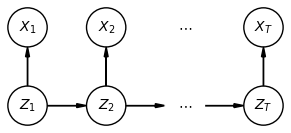}
  \caption{HMM Diagram}
  \label{fig:hmm}
\end{figure}

\begin{fact} $\Stat$ is a sub-category of the graph category.
\end{fact}

\section{Methods}

In this section, we shall exibit the general methods to enhance the basic models. They are treated as the functors/constructors of the category $\Stat$.

\subsection{Traditional methods}

\begin{definition}[Variational model]
  A \textit{variational distribution} $Q$ is always an apprximation of a given distribution $P$. A variational model of the unsupervised model $P(x,z)$, is represented by a tuple of the model distribution $P(x,z)$ and the variational distribution $Q(z|x)$ of the conditional distribution $P(z|x)$, $(P(x,z), Q(z|x))$. As well known, the estimator of the model is the maximal of the variational lower boundness,
  \begin{align}
   F(p,q;\{x_i\}):= & \sum_i F(p,q;x_i) \nonumber\\
   = & \sum_i \int_z q(z|x_i)\ln \frac{p(x_i,z)}{q(z|x_i)} dz
  \end{align}
\end{definition}

\begin{remark}
  An autoencoder is a tuple of an encoder $\phi:\mathcal{X}\to \mathcal{Z}$ and a decoder $\psi:\mathcal{Z}\to \mathcal{X}$, learning from the data where $\mathcal{X}$ and $\mathcal{Z}$ are the sample space and the encoding/latent space respectively. When the an encoder and the decoder are stochastic, represented by the conditional distributions $P(z|x)$ and $P(x|z)$ of the latent variable model $P(x,z)$, it is called the probabilistic/stochastic autoencoder. Now if the encoder $P(z|x)$ is replaced with the variational distribution $Q(z|x)$, then it is named the \textit{variational autoencoder (VAE)}\cite{kingma2013}.
\end{remark}

Following are the traditional methods, as the constructors, mapping the model $P(X,Y)$ or $P(x,z)$ to

\begin{table}[htbp]
\centering
\caption{Summary of the methods and their descriptions. ``-'' means that the author has not designed the symbol.}
\begin{adjustbox}{max width=\textwidth}
\begin{tabular}{|l|l|p{0.5\linewidth}|}
\hline
\textbf{Technique} & \textbf{Symbol of Constructor} & \textbf{Description} \\
\hline
Kernel trick & K & $P(\phi(X),Y)$ \\
\hline
Localization (refer to \autoref{sec:local}) & Loc & $K(x_*,X)\ln P(X)$ where $x_*$ is the target point \\
\hline
Hierarchical model & H & $P(x,z_1,\cdots, z_n,Y)$ where $X\to z_1\to \cdots\to z_n\to Y$, forms Markov chain usually and $z_1, \cdots, z_n$ are hidden \\
\hline
Variational trick & V & $(P(x,z), Q(z|x))$ where $Q(z|x)$ is the variational distribution \\
\hline
Time Series & TS & $z(n)$, where $z(0)=X, z(T)\sim N(0,1)$ \\
\hline
\end{tabular}
\end{adjustbox}
\label{tab:methods}
\end{table} 

Of course, adding conditional variational and Bayes method (adding priori) is one of such methods, have been introduced above.

\subsection{Neural models}

The neural models are defined as the models equipped with artificial neural networks. Usually, they are obtained through replacing the linear structures by the neural networks.

A neural method provides an implementation of a conditional distribtion, instead of changing the basic structure of the model. Denote $\NN(M)$ for any possible application of neural networks to model $M$. The following arrow is used to describe how the neural networks are assembled into a statistical model: 
\begin{equation}
 \NN: P(y|x)\to Y=f(X)
\end{equation} 
where $f$ represents any neural network, and $P(y|x)$ is any conditional distribution of two rvs in a model.

\begin{itemize}
  \item Feedforward Neural Models: the hierachical structure of linear  neural network layers
  \item RNN/LSTM: as a conditional HMM implemented by the neural network, that could be denoted as $\NN(\C(HMM))$
  \item Neural Autoencoder: NLPCA \cite{kramer1991}, integrating neural network into the autoencoder
  \item Probabilistic/Stochastic Neural Autoencoder: Kingma-Welling's Variational Autoencoder\cite{kingma2013}
  \item Stochastic neural network(SNN): neural network with Dropout(the stochastic perturbation affacts the weights of the layers, or the outputs of the layers)
  \item Hierarchical VAE: Diffusion Model/Consistency Model
\end{itemize}

Normalization Flow is a method to simulate a distribution directly, by the following fact.
\begin{fact}
For any distribution $P$, we can find a function $f$, that satisfies
\begin{equation}
  x\sim P \iff f(z), z\sim Q
\end{equation} 
where $Q$ is an \textit{easy-to-sample-from} distribution such as standard Gaussian distribution, uniform distribution.
\end{fact}

In practice, we always let $f$ be a neural network.

\section{To create advanced models}

The main role of the category-theoretical tool is to create or describe statistical learning models by applying functors. We give some examples as follows.

\subsection{Take VAE as an example}

The VAE proposed by \cite{kingma2013} is a special model of stochastic autoencoder implemented by neural networks.

Following is the path to produce a general VAE. 
\begin{align}
  P(x) & \to P(x,z)\nonumber\\
  & \to (P(x,z),Q(z|x =x)\sim N(g(x),h(x)))\nonumber\\
  & \to (P(x,z),Q(z|x =x) = g(x)+\xi h(x)),\xi\sim N(0,1)
\end{align}

Write it in the style of the composition of functors (informally) 
\begin{equation}
VAE(f,g,h) = \mathrm{Rep}\circ \mathrm{V} \circ\mathrm{LV}(P(X)) \comma
\end{equation}
regarding functions $f,g,h$ as parameters.

A VAE could be implemented by the following SNN (with a regularizing term): 
\begin{equation}
  y \sim f(g(x)+h(x)\xi) \comma
\end{equation} 
trained through self-supervised manner with data $\{(x_i,x_i)\}$, where $f,g,h$ are all neural networks, $\xi\sim N(0,1)$ is the perturbation variable of the hidden layer $g(x)$.

The network structure for VAE can be referred to as a \textit{stochastic neural network} due to its hidden layer containing stochastic perturbations. This implies an interesting fact: stochastic perturbations in the hidden layer of a neural network can result in better learning outcomes.

\subsection{Take RNN as an example}\label{sec:rnn}

Essentially, an RNN is a special type of conditional HMM, which is constructed as, 
\begin{equation*}
  P(y)\to \cdots \to P(y^*,z^*|x^*)
\end{equation*} 
Then it will be implimented by the neural networks: 
\begin{equation*}
  \to y_t\sim Net(x_t,z_{t-1}),z_{t}\sim Net(x_t,z_{t-1}),t=1,2,\cdots
\end{equation*}

Finally, an RNN could be formulated as, 
\begin{equation}\label{rnn}
  RNN = \mathrm{NN \circ TS_z\circ C\circ LV}(P)
\end{equation}

In the same manner, it is easy to create time series version of VAE, that is $\TS_z(\V(\C_x(P)))$.

RNNs struggle with long-term dependencies due to the vanishing gradient problem, where gradients become smaller and smaller as they propagate through the network. LSTM addresses this issue by adding memory cells to the RNN architecture. In LSTM there is tow types of hidden variables, but combining these two variables, we will find it never change the basic structure of RNN.

% ### autoencoder
% $(P(x,z), Q(z|x))     o (P(X^*,Z^*), Q(Z^*|X^*))$
% 
% Such AE encodes a whole sequence instead the elements in the seq independetly, namely the encoding result of an element depends on its position, and hence we say that it is dynamic. Anyway it is a common/static AE on $\mathcal{X}^*$.

\subsection{To understand the Large Language Models}

In this section, we will construct Large Language Models (LLMs) by the category-theoretic languange.

Recall the brief history of the development of LLMs through the following three models:
\begin{itemize}
  \item RNN/LSTM model of languages \cite{belinkov2017}
  \item BiRNN/BiLSTM (Bidirection LSTM) / ELMo (Embeddings from Language Models)\cite{peters2018}
  \item BERT\cite{devlin2019}
\end{itemize}

The RNN model for natural language is no different from the ordinary RNN (see \autoref{sec:rnn}), except that the input and output are both word embeddings. Now let us understand the other two structures.

\subsubsection{ELMo}

ELMo (Embeddings from Language Models) is a deep learning-based language model that learns contextualized word representations by jointly training a bidirectional language model.

\begin{figure}[h]
  \centering
  \includegraphics[scale=0.7]{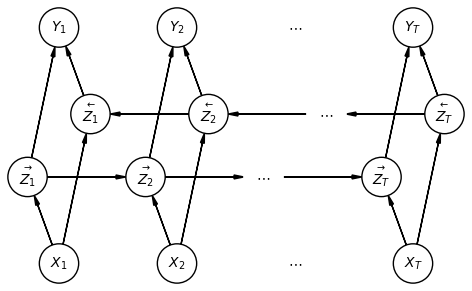}
  \caption{The structure of ELMo. A more detailed figure can be found in \cite{devlin2019}}
  \label{fig:elmo}
\end{figure}

\begin{remark}
\autoref{fig:elmo} is not a typical GPM, only used after the introduction of BiLSTM.
\end{remark}

There is an issue, but not so serious. The objective of ELMo seams not a true joint distribution. If there is not any common parameter in the two log likelihoods, then optimize the objective is equivalent to optimize them seperatively. However there are shared parameters in two distributions. So it is not a joint distribution any longer, instead it is just a loss function for a statistical desicion. If it is normalized (becoming a true distribution.), then we speculate that it will perfume better then the unnormlized one, and call ELMo with such distribution the normalized ELMo.

\subsubsection{BERT}

\begin{figure}[h]
  \centering
  \includegraphics[scale=0.7]{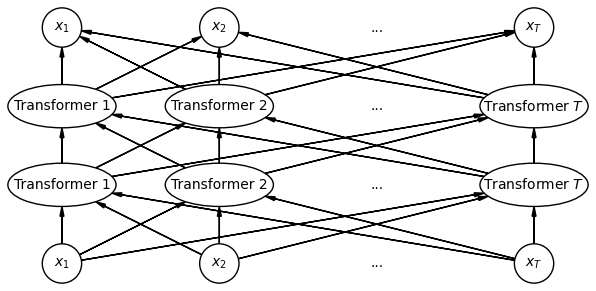}
  \caption{The structure of BERT (a whole layer is the local model of Transformer \autoref{df:transformer})}
  \label{fig:bert}
\end{figure}

BERT is just the distribution of the token sequence instead of introducing Markov structure of the sequence. Since BERT casts off the Markov chain of the states, it must take the hierachical form, otherwise, it will form a cycle as a bidirection model, that is illegal in the directed graph model.

In the next section, we will give the way to construct the Transformer.

\section{Statisitical Decision}

Many machine learning models are not described via an explicit distribution, such as Kmeans clustering and \autoref{df:ml}. In this case, we construct a model only with a loss function based on the sample or its corresponding optimization problem. The localization method is subsequently addressed.

\subsection{The definition of the statisitical decision}

\begin{definition}[Statisitical Decision\cite{shao2007}]
  Assume that $l(x,\theta)$ is the loss on the single point $x$ of parameter $\theta$. The loss/risk of the model is defined as the expectation $El(X,\theta)$, and replaced by the empirical risk 
  \begin{equation}
  J(\theta):= \frac{1}{N}\sum_i l(x_i,\theta) \comma
  \end{equation}
if the real distribution is unknown. The statisitical decision is just the following optimization problem,
  \begin{equation}\label{erm}
  \min_\theta J(\theta) \comma
  \end{equation}
  that is the principle of the \textit{empirical risk minimization (ERM)}.
\end{definition}

We use $l(x,\theta)$ to denote a statisitical decision.

Obviously, the MLE is a special case of statisitical decision or ERM whose loss is the negative likelihood. And it is not difficult to define statisitical decision for statistical learning models. For example, the statisitical decision for regression models could defined as the following optimization problem,
\begin{equation}\label{dec-reg}
\min_{\theta} \big\{J(\theta):=\frac{1}{N}\sum_i l(x_i,y_i,\theta)\big\} \comma
\end{equation}
where $\{(x_i,y_i)\}$ is the sample. \eqref{dec-reg} can be created by applying the conditionalizing constructor $\C$ to the statisitical decision $l$, deonted as $\C(l)$.

Another example is that the statisitical decision for latent variable model could defined as the following optimization problem, as a generalization of MoMLE,
\begin{equation}\label{dec-lvm}
\min_{\theta,\{z_i\}} \big\{J(\theta,\{z_i\}):=\frac{1}{N}\sum_i l(x_i,z_i,\theta)\big\} \comma
\end{equation}
where $\{x_i\}$ is the sample and $z_i$ is unobservable. \eqref{dec-lvm} can be constructed as $\LV_z(l)$.

\subsection{Decision Tree}

People think that decision trees are not common statistical learning models. However, after carefully examining decision tree algorithms such as ID3, a decision tree is a recursive decision based on the following parametric model, 
\begin{equation}\label{id3}
  P(y|x,j)=P(y|x_j), j=1,\cdots, p
\end{equation}
where $X_j$ is the $j$-th component of $X$. As known, the MLE of \eqref{id3} is equivalent to the principle of maximum information gain, which determines the splitting rule of the tree. We create a construtor $\Tr(J)$ to depict a decision tree based on any statisitical decision $J(\theta)$ including the MLE like \eqref{id3}. For example, the linear tree \cite{gama1999} is represented by $\Tr(P(y|x))$ where $P(y|x)$ is a linear classifier.

% We can extend the decision tree as follows:
% \begin{itemize}
%   \item use various models/loss functions for each splitting.
%   \item use a soft splitting rule which allocates samples to different subtrees with certain probabilities or weights, instead by a hard one.
% \end{itemize}

We denote the decision tree as $\Tr(M)$, where $M$ is a model to determine the splitting rules.

\subsection{Local statistical decision}\label{sec:local}

It is more natural to describe the localizaiton method from the view of statistical decision.

\begin{definition}[Local statistical decision/Local model]\label{df:local-sd}
  Given a loss function $l$ and a kernel $K(x,y)$, we say the following optimization problem is a local statistical decision at the target point $x_*$,
  \begin{equation}\label{local-sd}
  \min_\theta \frac{1}{N}\sum_i K(x_*,x_i)l(x_i,\theta)
  \end{equation}
  where $l(x,\theta)$ is a loss at the single point $x$, $\{x_i\}$ is the sample. We would like to say that \eqref{local-sd} is built through appling the localization trick to the $l$.
\end{definition}

the kernel $K$ in local model is more free then that in the kernel method. In most cases, it is only expected to be non-negative.

since it is not required to be symmetric, A kernel $K(x,x')$ can be represented by two different feature mappings $\phi,\psi: \mathcal{X}\to H$ where $H$ is a Hilbert space, namely 
\begin{equation*}
  K(x,x')=\langle\phi(x),\psi(x')\rangle
\end{equation*}
When $H$ is a low dimensional Euclidean space, the feature mappings are also called the embedding or the distributional/real representation. More generally, we define 
$$K(x,x')=F(\phi(x),\psi(x'))$$
such as $K(x,x')=e^{\langle\phi(x),\psi(x')\rangle}$.

The solution of \eqref{local-sd} is related to $x_*$, thus denoted as $\hat{\theta}(x_*)$.

The total loss on the sample is 
\begin{equation}
  J(K)=\sum_i l(x_i,\hat{\theta}(x_i;K))
\end{equation} 
where the expressions $J(K)$ and $\hat{\theta}(x_i;K)$ stress that the loss is also related to the kernel $K$.

Following definition reflects the original idea of localization.

\begin{definition}[Local model for machine learning]\label{df:local-ml}
  Given a machine learing model $y\sim f(x,\theta)$, we define its localized model as
  \begin{equation}
  \min_\theta \frac{1}{N}\sum_i K(x_*,x_i)|y_i-f(x_i,\theta)|^2
  \end{equation}
  or for some purposes,
  \begin{equation}\label{local-ml}
  \min_\theta \frac{1}{N}\sum_i K(x_*,y_0,x_i,y_i)|y_i-f(x_i,\theta)|^2
  \end{equation}
\end{definition}

Finnally, we introduce $Loc(M)$ (or $Loc(l)$) to describe a local model whose fundamental model is $M$ or (determined by a loss function $l$).

\subsection{Local average}

The terminal goal of the regression is to calculate the conditional expection, 
\begin{equation}\label{local-average}
  E(y|x)\approx \sum_{x_i\in U_x} y_i
\end{equation} 
where $U_x$ is a certain neighorhood of $x$.

On the basis of \eqref{local-average}, we give the following concept.

\begin{definition}[Local average/Local regression]
A local average of the sample $\{(x_i,y_i)\}$ on target var $x_*$, is defined as,
  \begin{equation}
    \hat y(x_*):=\sum_i K(x_*,x_i)y_i/\sum_i K(x_*,x_i) \comma
  \end{equation}
where the sample space $\mathcal{Y}$ of $y$ is assumed to be the real numbers $\mathbb{R}$ or the Euclidean space $\mathbb{R}^p$.
\end{definition}

\begin{fact} 
  Any local regression is reduced to the local average, approximately.
\end{fact}

Especially, we have the local average of the sample $\{(x_i,x_i)\}$, also named the self local average of $\{x_i\}$.

\begin{definition}
  The \textit{(self) local average mapping/transform} (or \textit{mean shifting}) is defined as,
  \begin{equation}\label{meanshift}
    m(x_*):=\sum_i K(x_*,x_i)x_i/\sum_i K(x_*,x_i):\mathcal{X}\to\mathcal{X}.
  \end{equation}
  Mapping on the sample $X=\{x_i\}$, we have
  \begin{gather}\label{meanshift2}
    m(x_i):=\sum_j K(x_i,x_j)x_j/\sum_j K(x_i,x_j):\mathcal{X}^N\to\mathcal{X}^N \comma \\
    \text{i.e.} ~ m(X)=\tilde{K}X \comma
  \end{gather}
  where $\tilde{K}$ is the is the normlization of the kernel matrix $K(X,X)$.
\end{definition}

We can call \eqref{meanshift} ``the self-local average'', partly in the reason that it has the same form with ``the query-key-value model'' of the self-attention \cite{vaswani2017}.

\begin{remark}
  The famous MeanShift algorithm \cite{comaniciu2002} is the iteration of the mapping $m$.
\end{remark}

\begin{remark}
  The local linear embedding (LLE)\cite{roweis2000} is indeed a self-local average with learnable kernel $K$.
\end{remark}

For classification, we have the following analogy.

\begin{definition}[Local mode/Local classification]
A local mode of the sample $\{(x_i,y_i)\}$ on target var $x_*$, is defined as,
  \begin{equation}
    \hat y(x_*):=\arg\max_k\sum_{y_i=k} K(x_*,x_i) \comma
  \end{equation}
where the sample space $\mathcal{Y}$ of $y$ is a finite set, taking the values $k=1,\cdots,K$.
\end{definition}

A classisification model can be transformed to a regression model, by embedding the discrete space $\mathcal{Y}$ into a continuous space, as in word embedding.

\subsection{Local latent variable model}

It is possible to apply the localization trick into the latent variable models. As we all know, the common method to solve the latent variable model is based on the variational lower bound, instead of the likelihood function, as the loss function of the model. Hence it is able to define a ``local variational lower bound''.

Here, we consider the autoencoder as a special latent variable model. We realize that the autoencoder should also be linear in the local domain, and the most typical linear latent variable model/autoencoder is PCA. Therefore, we can directly establish a ``local PCA'' model, denoted as $\Loc(PCA)$.

As we all know, the optimization of the latent variable model is the variational lower bound, which is to replace the likelihood function with the variational lower bound as the loss function. We can define the ``local variational lower bound''. Here, we consider the autoencoder as a special latent variable model. We realize that the autoencoder should also be linear in the local range, and the most typical linear autoencoder is PCA.

Obviously, the autoencoder can also be constant locally, and defined as the following optimization problem.
\begin{equation}\label{local-ae}
  \min_{z}\sum_i K(x_*,x_i)\|x_i-z\|_2^2
\end{equation} 
It is easy to see that the solution to \eqref{local-ae} is $\hat{z}(x_*)=\sum_i K(x_*,x_i)x_i/\sum_i K(x_*,x_i)$, that is the local average. This implies the following fact.

\begin{fact}
The local dimensionality reduction, the local clustering, and the local average are equivalent concepts in some sense.
\end{fact}

\subsection{Local statistical decision for autoregression}

Now we delve into the local statistical decision for the autoregression or discrete dynamics (of 1-order), that is represented as follows:
\begin{equation}\label{ar}
  x_{t+1}\sim f(x_t,t), x_t\in\mathcal{X},t=1,2,\cdots
\end{equation}
where $\mathcal{X}$ is assumed to be Euclidean space $\mathbb{R}^p$, otherwise we should embed it into the real space first. It is indeed a stochatic process, but in a supervised learning form (\autoref{df:ml}).

According to the \autoref{df:local-sd} or \autoref{df:local-ml}, the local model of \eqref{ar} could be expressed as the following local average,
\begin{equation}\label{local-ar}
  \hat x_{t+1}:= \sum_s K(x_t,t,x_{s},s)x_{s+1}/\sum_s K(x_t,t,x_s,s)
\end{equation}
treating the tuple $(x_{t},t)$ as the input and $x_{t+1}$ as the outputs. \eqref{local-ar} could be written as $\Loc(x_{t+1}\sim f(x_t,t))$ roughly. Let us rewrite it in the following definition.

\begin{definition}[Local autoregression/Dynamical system/Local time-series]
  The local model of \eqref{ar} is called the local autoregression/time-series, expressed as
  \begin{equation}\label{serialized-self-local-average}
  \hat x_{t} = m(x_{t},t):= \sum_s K(x_t,t,x_s,s)x_{s}/\sum_s K(x_t,t,x_s,s)
  \end{equation}
  where the time $t$ could theoretically take any type of values, and $x_t$ is the series on it. We prefer using the sequential form: $\{\hat x_{t}\}=m(\{x_{t}\})$ that is a mapping from the sequential space $\mathcal{X}^*$ to itself.
\end{definition}

\begin{fact}
 The local average \eqref{serialized-self-local-average} can approximate any autoregression model or discrete dynamics (of any order). In fact, it is the result of the construction of local HMM, write, $\Loc(HMM)$.
\end{fact}

It follows from \eqref{local-ml} that the local autoregression is a local regression of $\mathcal{T}\to \mathcal{X}$, write, $\Loc(x_t=f(t))$. We conclude that,
\begin{fact}
  All local models are equivalent in some sence, and they are all local average.
\end{fact}

When the kernel is unrelated to the positions, i.e., $K$ is designed as $K(x_t,x_s)$, it is reduced to ordinary self-local average\cite{lee2019}. Otherwise, \eqref{serialized-self-local-average} could be termed as ``serialized self-local average'', providing a strict mathematical interpretation for self-attention.

A simple design of $K$ is $K_1(x_t,x_s)K_2(t,s)$ or $K_1(x_t,x_s)K_2(t-s)$, where $K_1$ represent the graphic-dependence of elements in $\mathcal{X}$ statically and $K_2$ represents the ``position-encoding''. It is essentially the weighted moving average.

Hence, what the localization really dose is to transform the time-dependency to graphic-dependency that is no limited on long-dependency theoretically.

\begin{definition}[self attention with absolute positional embedding\cite{vaswani2017}]
  The self attention is a local dynamical system with kernel $K(x_t+p(t),x_s+p(s))$ where $p(t)$ is the position-encoding of $t$.
\end{definition}

Strictly, the local dynamical system should learn $K$ to implement the self-attention, that is to solve the following optimization problem: 
\begin{equation}
  \min_K J(K):=\|X-\tilde{K}X\|^2_{F}
\end{equation} 
where $X=\{x_{t}^{(j)}\}_{tj}$ and $\tilde{K}$ is the normlization of $K=\{K(x_t,t,x_s,s)\}_{st}$.

\subsection{Local-local model}

A local model is determined by a local loss $J(x_*,\theta;K)$. It is also a loss function at a single sample point. Based on this, we can further construct a local loss as follows:
\begin{align}\label{local-local}
J'(x_*,\theta;K,K') &= \sum_kK'(x_*,x_k)J(x_k,\theta;K) \nonumber\\
 & =\sum_{kj}K'(x_*,x_k)K(x_k,x_j)l(x_j,\theta)
\end{align}
In matrix form, it can be expressed as $J'(X,\theta)=K'(X,X)K(X,X)l(X,\theta)$.

We refer to \eqref{local-local} as the ``local-local loss'', and the corresponding decision model as the ``local-local model''.

\eqref{local-local} is equivalent to the local loss of the (composite) kernel $(x,y)\mapsto\sum_{k}K'(x,x_k)K(x_k,y)$. Therefore, the local-local model is still a local model. Similarly, the local-local average can be constructed, that is the local average of the composition of normalized kernels:
\begin{equation}\label{local-local-ave}
\hat{x}_i=\sum_{kj}{\tilde{K}'}_{ik}\tilde{K}_{kj}x_j, \text{i.e.}~ \hat{X}=\tilde{K}'KX
\end{equation}

To prevent a hierarchical local model from deteriorating into an ordinary local model, we insert a nonlinear function between two kernels, mimicking the feedforward neural network:
\begin{equation}\label{local-local-ave-f}
\hat{x}_i=\sum_{k}\tilde{K}'_{ik}f(\sum_j\tilde{K}_{kj}x_j),f:\R^p\to\R^p, \text{i.e.}~ \hat{X}=\tilde{K}'f(KX)
\end{equation}
where $\tilde{K}'$ is the normalization of $K(X',X')$, $X'=f(\tilde{K}X)$. It is more reasonable to rewrite \eqref{local-local-ave-f} as
\begin{equation}
\hat X = m_{K'} (f(m_{K} (X)))
\end{equation}
by the local average mapping.

Based on \eqref{local-local-ave-f}, we can further construct ``local-local-local models'' and so on. We collectively refer to these models as ``hierarchical local models''. We coin a symbol $\Loc^L$, for such construction, where $L$ is the number of the kernels, and $\Loc^1=\Loc$. In fact, Transformer, the most popular large model structure, is a such hierarchical local model.

\begin{definition}[Transformer]\label{df:transformer}
The encoder part of the Transformer can be represented as the following sequence mapping:
\begin{equation}
  \hat X = m_{K_L}(f_{L-1}( \cdots (f_1(m_{K_1} (X)))))
\end{equation}
where $m_{K_l}$ is the (sequential) local average of kernel $K_l$, and $f_l$ represents a feed-forward neural network. 
\end{definition}

\cite{vaswani2017} suggests $L=6$, thus Transformer $=\Loc^6(y_t\sim f(x_t,t))$. (Here, we have not yet introduced the multi-head structure.)

\autoref{df:transformer} only provides the structure of the Transformer encoder, and according to the suggestion in \eqref{local-ar}, the decoder part is redundant.

\begin{remark}
Note that, in \autoref{df:transformer}, the Transformer is defined to be one layer of transformer blocks in the architecture of BERT (\autoref{fig:bert}), instead of a single transformer block.
\end{remark}

\subsection{Tied models}

Inspired by ELMo as a BiLSTM, we propose a new model or constructor, called \textit{tied model}, denoted as the following pair,
\begin{equation}
  \mathrm{Tie}(P(X|\theta_1,\theta_0), P(X|\theta_2,\theta_0))
\end{equation} 
where $\theta_0$ is the common parameter and $\theta_i,i=1,2$ are the private parameters.

The estimation of parameter is to maximize ``the tied log-likelihood'',
\begin{equation}\label{tied}
\alpha\ln P(X|\theta_1,\theta_0)+(1-\alpha)\ln P(X|\theta_2,\theta_0)
\end{equation} 
where $0\leq \alpha\leq 1$ is a constant.

As mentioned above, \eqref{tied} is not a real likelihood, but a special loss function. The tied model is judged to be a product of experts \cite{hinton2002} but unnormalized and could be regarded as a statistical decision.

\subsection{Markov decesion process}

The \textit{Markov decesion process (MDP)} is a subcategory of time series $\mathcal{P}^*(\{x_t\})$. As a decistion model, it has the loss function in the following form,
\begin{equation}
  l(\theta;x_1,\cdots, x_T) = \sum_{t=1}^{T-1} l(\theta;x_{t+1}|x_t)
\end{equation}
where $l(\theta;x_{t+1}|x_t)$ represents a transition loss in the $t$-th step.

Reinforcement learning is indeed a special type of Markov decision process, where the policy governing actions from a given state serves as the parameter, and the conventional notion of loss is substituted with the concept of reward, while the true value of reward $l(\theta;x_{t+1}|x_t)$ is obtained through the interaction with environment.

We can create a latent variable model for MDP, as $\LV_{\{z_t\}}(MDP)$, and get the so-called partially observable MDP.

\section{Novel paradigms of machine learning}\label{sec:paradigms}

Numerous novel paradigms have been developed in the field of machine learning, including:
\begin{itemize}
  \item Reinforcement learning
  \item Ensemble Learning
  \item Transfer Learning
  \item Incremental Learning (Continual Learning, On-line Learning)
  \item Life-Long Learning (as a more general form of incremental learning)
  \item Meta Learning
  \item Multi-task Learning
  \item Dual Learning
\end{itemize}

We do not plan to discuss the paramdigms here. For the transfer Learning, several potential statistical definitions include:
\begin{itemize}
  \item $(P(X^{s}|\theta_1), P(X^{t}|\theta_2)), \theta_1,\theta_2\sim P(\theta|\alpha)$
  \item $(P(X^{s}|\theta_1,\theta_0), P(X^{t}|\theta_2,\theta_0))$
  \item $(P(\phi(X^{s})|\theta), P(\phi(X^{t})|\theta))$ where $\phi$ is learnable for making $\phi(X^{s})$ and $\phi(X^{t})$ have the same distribution.
\end{itemize}

Here, $X^{s}$ represents the source domain, and $X^{t}$ represents the target domain.

\section{Future Works}

This paper's fundamental objective is to employ category theory as a unified language to describe statistical learning models, and even to construct novel models using this framework. Currently, the number of statistical learning models described in this way is limited, but the author believes that other models newly developed, especially the LLMs, can also be described similarly. If successful, future efforts will be directed towards establishing a more systematic category-theoretical framework for statistical learning as well as the novel paradigms of machine learning listed in \autoref{sec:paradigms}. Guided by category theory, we will further invent new models.

It is worth mentioning that we have strictly defined the local decision-making model, also regarded as a constructor, and based on this, we have constructed the Transformer.

The facts stated in thet paper are not strict, as informal theorems, with the phrase ``in some sense''.

\bibliographystyle{plain}
\bibliography{references}
\end{document}